%% file: arxiv-mmNMT.tex
\documentclass{article}

\usepackage[preprint]{modified_nips2018}

\usepackage{hyperref}
\input{preamble}

\title{Moment Matching Training for Neural Machine Translation: A Preliminary Study}

\author{Cong Duy Vu Hoang\thanks{\noindent Work done during an internship at NAVER Labs Europe, Grenoble, France}  \\
  Computing and Information Systems \\
  University of Melbourne, Australia \\
  {\tt vhoang2@student.unimelb.edu.au} \\\And
  Ioan Calapodescu \qquad\qquad Marc Dymetman\\
  NAVER Labs Europe, Grenoble, France\\
  {\tt ioan.calapodescu@naverlabs.com} \\
  {\tt marc.dymetman@naverlabs.com}}

\date{}

\begin{document}

\maketitle


\input{abstract}

\input{intro_motiv}

\input{mm_model}

\input{rel_works}

\input{exp}

\input{concl}

\section*{Acknowledgments}
Cong Duy Vu Hoang would like to thank NAVER Labs Europe for supporting his internship; and Reza Haffari and Trevor Cohn for their insightful discussions. Marc Dymetman wishes to thank Eric Gaussier and Shubham Agarwal for early discussions on the topic of moment matching. 

\bibliography{reference-mm}
\bibliographystyle{acl_natbib}

\appendix

\input{supp}

\end{document}

%% file: abstract.tex
\begin{abstract}
In previous works, neural sequence models have been shown to improve significantly if external prior knowledge can be provided, for instance by allowing the model to access the embeddings of explicit features during both training and inference. 
In this work, we propose a different point of view on how to incorporate prior knowledge in a principled way, using a \emph{moment matching} framework. In this approach, the standard local cross-entropy training of the  sequential model is combined with a \emph{moment matching} training mode that encourages the equality of the expectations of certain predefined features between the model distribution and the empirical distribution. In particular, we show how to derive unbiased estimates of some stochastic gradients that are central to the training, and compare our framework with a formally related one: policy gradient training in reinforcement learning, pointing out some important differences in terms of the kinds of prior assumptions in both approaches. 
Our initial results are promising, showing the effectiveness of our proposed framework. 
\end{abstract}

%% file: intro_motiv.tex
\section{Introduction and Motivation}

Standard training of neural sequence to sequence (seq2seq) models requires the construction of a cross-entropy loss \citep{Sutskever:2014:SSL:2969033.2969173,2015arXiv150600019L}. 
This loss normally manipulates at the level of generating individual tokens in the target sequence, hence, potentially suffering from label or observation bias \citep{wiseman-rush:2016:EMNLP2016,DBLP:journals/corr/PereyraTCKH17}. 
Thus, it might be difficult for neural seq2seq models to capture the semantics at sequence level.   
This may be detrimental when the desired generated sequence may be missing or lacking some desired properties, for example, avoiding repetitions, preserving the consistency between source and target length ratio, or satisfying biasedness upon some external evaluation measures such as ROUGE \citep{W04-1013} and BLEU \citep{Papineni:2002:BMA:1073083.1073135} in summarisation and translation tasks, respectively; or avoiding omissions and additions of semantic materials in natural language generation; etc. 
Sequence properties, on the other hand, may be associated with prior knowledge about the sequence the model aims to generate. 

In fact, the cross-entropy loss with sequence constraints is intractable. 
In order to inject such prior knowledge into seq2seq models, methods in reinforcement learning (RL) \citep{Sutton1998ReinforcementLA} emerge as reasonable choices.  
In principle, RL is a general-purpose framework applying for sequential decision making processes. 
In RL, an agent interacts with an environment $\mathcal{E}$ over  a certain number of discrete timesteps \citep{Sutton1998ReinforcementLA}. 
The ultimate goal of the agent is to select any action according to a policy $\pi$ that maximises a future cumulative reward. 
This reward is the objective function of RL guided by the policy $\pi$, and is defined specifically for the application task. 
Considering seq2seq models in our case,  an action of choosing the next word prediction is guided by a stochastic policy and  receives a task-specific reward with a real value return. 
The agent tries to maximise the expected reward for $T$ timesteps, e.g., $\mathcal{R} = \sum_{t=1}^{T} r_t$. 
The idea of RL  has recently been applied to a variety of neural seq2seq tasks. 
For instance, \cite{DBLP:journals/corr/RanzatoCAZ15} applied this  idea to abstractive summarisation with neural seq2seq models, using the ROUGE evaluation measure \citep{W04-1013} as a reward. 
Similarly, some success also has been achieved for neural machine translation, e.g., \citep{DBLP:journals/corr/RanzatoCAZ15, NIPS2016_6469, NIPS2017_6622}. 
\cite{DBLP:journals/corr/RanzatoCAZ15} and \cite{NIPS2017_6622} used BLEU score \citep{Papineni:2002:BMA:1073083.1073135} as a reward function in their RL setups; whereas \cite{NIPS2016_6469} used a reward interpolating the probabilistic scores from reverse  translation  and language models. 

\subsection{Why Moment Matching?}
The main motivation of moment matching (MM) is to inject prior knowledge into the model which takes the properties of whole sequences into consideration. 
We aim to develop a generic method that is applicable for any seq2seq models. 

Inspired from the method of moments in statistics,\footnote{https://en.wikipedia.org/wiki/Method\_of\_moments\_(statistics)} we propose the following moment matching approach. 
The underlying idea of moment matching is to seek optimal parameters reconciling two distributions, namely: one from the samples generated by the model and another one from the empirical data. 
Those distributions aim to evaluate the generated sequences as a whole, via the use of feature functions or constraints that one would like to behave similarly between the two distributions, based on the encoding of the prior knowledge about sequences. 
It is worth noting that this proposed moment matching technique is not stand-alone, but to be used in alternation or combination with standard cross-entropy training. 
This is similar to the way RL is typically applied in seq2seq models \citep{DBLP:journals/corr/RanzatoCAZ15}. 

Here, we will discuss some important differences with RL, then we will present the details on how the MM technique works in the next sections. 

The \emph{first difference} is that RL assumes that one has defined some reward function $\mathcal{R}$, which is done quite independently of what the training data tells us. 
By contrast, MM only assumes that one has defined certain features that are deemed important for the task, but one then relies on the actual training data to tell us how to use these features. 
One could say that the ``arbitrariness'' in MM is just in the choice of the features to focus on, while the arbitrariness in RL is that we want the model to get a good reward, even if that reward is not connected to the training data at all.

Suppose that we are in the context of NLG and are trying to reconcile several objectives at the same time, such as (1) avoiding omissions of semantic material, (2) avoiding additions of semantic material, (3) avoiding repetitions \citep{W17-5519}. 
In general, in order to address this kind of problem in an RL framework, we need to ``invent'' a reward function based on certain computable features of the model outputs which in particular means inventing a formula for combining the different objectives we have in mind into a single real number. 
This can be a rather arbitrary process, and potentially it does not guarantee any fit with actual training data.  
The point of MM is that the only arbitrariness is in choosing the features to focus on, but after that it is actual training data that tells us what should be done.

The \emph{second difference} is that RL tries to maximize a reward, and is only sensitive to the rewards of individual instances, while MM tries to maximize the fit of the model distribution with that of the empirical distribution, where the fit is on specific features.

For instance, this difference is especially clear in the case of language modelling where RL will try to find a model that is strongly peaked on the $x$ which has the strongest reward (assuming no ties in the rewards), while MM will try to find a distribution over $x$ which has certain properties in common with the empirical distribution, e.g., for generating diverse outputs.  
For language modelling, RL is a strange method, because language modelling requires the model to be able to produce different outputs; for MT, the situation is a bit less clear, in case one wanted to argue that for each source sentence, there is a single best translation; but in principle, the observation also holds for MT, which is a conditional language model. 

%% file: mm_model.tex
\section{Proposed Model}
In this section, we will describe our formulation of moment matching for seq2seq modeling in detail. 

\subsection{Moment Matching for Sequence to Sequence  Models}
Recall the sequence-to-sequence problem whose goal is to generate an output sequence given an input sequence. 
In the context of neural machine translation - which is our main focus here, the input sequence is a source language sentence, and the output sequence is a target language sentence. 

Suppose that we are modeling the target sequence $\vy=y_1,\ldots,y_t,\ldots,y_{|\vy|}$ given a source sequence $\vx=x_1,\ldots,x_t,\ldots,x_{|\vx|}$, using a sequential process $\operatorname{p}_{\Theta}\left(\vy|\vx\right)$.  
This sequential process can be implemented via a neural mechanism, e.g., recurrent neural networks within an (attentional) encoder - decoder framework \citep{bahdanau:ICLR:2015} or a transformer framework \citep{NIPS2017_7181}. 
Regardless of its implementation, such a neural mechanism depends on model parameters $\Theta$. 

Our proposal is that we would like our sequential process to satisfy some moment constraints.  
Such moment constraints can be modeled based on features that encode prior (or external) knowledge or semantics about the generated target sentence. 
Mathematically, features can be represented through vectors, e.g., 
$\vPhi\left(\vy|\vx\right) \equiv \left(\phi_1\left(\vy|\vx\right),\ldots,\phi_j\left(\vy|\vx\right),\ldots,\phi_{m}\left(\vy|\vx\right)\right)$
, where $\phi_j\left(\vy|\vx\right)$ is the $j^{th}$ conditional feature function of a target sequence $\vy$ given a source sequence $\vx$, and $m$ is number of features or moment constraints. 
Considering a simple example where the moment feature is for controlling the length of a target sequence - which would just return a number of elements in that target sequence. 

\subsection{Formulation of the MM Objective Function}
In order to incorporate such constraints into the seq2seq learning process, we introduce a new objective function, namely the moment matching loss $\mathcal{J}_{MM}$. 
Generally speaking, given a  vector of features $\vPhi\left(\vy|\vx\right)$, the goal of moment matching loss  is to encourage the identity of the model average estimate,  
\begin{equation}
\begin{split}
\hat{\vPhi}_n\left( \Theta \right) \equiv \mathbb{E}_{\vy \sim \operatorname{p}_{\Theta}\left(.|\vx_n\right)}\left[ \vPhi\left(\vy|\vx_n\right) \right] \nonumber
\end{split}
\end{equation}
with the empirical average estimate, 
\begin{equation}
\begin{split}
\bar{\vPhi}_n \equiv \mathbb{E}_{\vy \sim \operatorname{p}_{\mathcal{D}}\left(.|\vx_n\right)}\left[ \vPhi\left(\vy|\vx_n\right) \right]; 
\end{split}
\end{equation}
where $\mathcal{D}$ is the training data; $\vx, \vy \in \mathcal{D}$ are source and target sequences, respectively; $n$ is the data index in $\mathcal{D}$. 
This can be formulated as  minimising a squared distance between the two distributions with respect to model parameters $\Theta$: 
\begin{equation}\label{eqn:mm_1}
\begin{split}
\mathcal{J}_{MM} \left( \Theta \right) &:= \frac{1}{N} \sum_{n=1}^{N} \norm{\hat{\vPhi}_n\left( \Theta \right) - \bar{\vPhi}_n}^2_2 \\
&= \frac{1}{N} \sum_{n=1}^{N} \norm{\mathbb{E}_{\vy \sim \operatorname{p}_{\Theta}\left(.|\vx_n\right)}\left[ \vPhi\left(\vy|\vx_n\right) \right] - \mathbb{E}_{\vy \sim \operatorname{p}_{\mathcal{D}}\left(.|\vx_n\right)}\left[ \vPhi\left(\vy|\vx_n\right) \right]}^2_2. 
\end{split}
\end{equation}
To be more elaborate, $\hat{\vPhi}_n\left( \Theta \right) \equiv \mathbb{E}_{\vy \sim \operatorname{p}_{\Theta}\left(.|\vx_n\right)}\left[ \vPhi\left(\vy|\vx_n\right) \right]$ is the model average estimate over the samples which are drawn i.i.d. from the model distribution $\operatorname{p}_{\Theta}\left(.|\vx_n\right)$ given the source sequence $\vx_n$, and $\bar{\vPhi}_n \equiv \mathbb{E}_{\vy \sim \operatorname{p}_{\mathcal{D}}\left(.|\vx_n\right)}\left[ \vPhi\left(\vy|\vx_n\right) \right]$ is the empirical average estimate given the $n^{th}$ training instance, where our data are drawn i.i.d. from the empirical distribution $ \operatorname{p}_{\mathcal{D}}\left(.|\vx\right)$.%

\subsection{Derivation of the Moment Matching Gradient}
We now show how to compute the gradient of $\mathcal{J}_{MM}$ in the equation~\ref{eqn:mm_1}, denoted as $\nabla_{\Theta} \mathcal{J}_{MM}$, which will be required in optimisation. 
We first define:
\begin{equation}
\begin{aligned}
\Gamma_{\Theta,n} \equiv \nabla_{\Theta} \left( \norm{ \Delta_n}^2_2  \right), \nonumber
\end{aligned}
\end{equation}
where $\Delta_n \equiv \hat{\vPhi}_n\left( \Theta \right) - \bar{\vPhi}_n$, then the gradient $\nabla_{\Theta} \mathcal{J}_{MM}$ can be computed as:
\begin{equation}\label{eqn:mm_2a}
\begin{aligned}
\nabla_{\Theta} \mathcal{J}_{MM} = \frac{1}{N} \sum_n \Gamma_{\Theta,n}.
\end{aligned}
\end{equation}

Next, we need to proceed with the computation of $\Gamma_{\Theta,n}$. By derivation, we have the following:
\begin{equation} \label{eqn:mm_6b}
\boxed{ \begin{aligned}
\Gamma_{\Theta,n} &= 2\sum_{\vy} \operatorname{p}_{\Theta} \left( \vy|\vx_n \right) \left\langle \hat{\vPhi}_{n}\left(\Theta \right) - \bar{\vPhi}_n, \vPhi\left( \vy|\vx_n \right) - \bar{\vPhi}_n  \right\rangle \nabla_{\Theta} \log \operatorname{p}_{\Theta} \left( \vy|\vx_n \right) \\
&= 2 \mathbb{E}_{\vy \sim \operatorname{p}_{\Theta}\left(.|\vx_n\right)} \left[ \left\langle \hat{\vPhi}_{n}\left(\Theta \right) - \bar{\vPhi}_n, \vPhi\left( \vy|\vx_n \right) - \bar{\vPhi}_n  \right\rangle \nabla_{\Theta} \log \operatorname{p}_{\Theta} \left( \vy|\vx_n \right) \right]. 
\end{aligned} }
\end{equation}

\begin{proof}
Mathematically, we can say that $\Gamma_{\Theta,n} $ is the gradient of the composition $\operatorname{F} \circ \operatorname{G}$ of two functions $\operatorname{F} \left( . \right) =  \norm{.}^2_2 : \mathbb{R}^{m}\rightarrow \mathbb{R}$ and $\operatorname{G} \left( . \right) = \hat{\vPhi}_n\left( . \right) - \bar{\vPhi}_n : \mathbb{R}^{|\Theta|}\rightarrow \mathbb{R}^m$. 

Noting that the gradient $\nabla_{\Theta}\left(\operatorname{F} \circ \operatorname{G}\right)$ is equal to the Jacobian $\mathcal{J}_{\operatorname{F} \circ \operatorname{G}}\left[\Theta\right]$, and applying the chain rule for Jacobians, we have:

\begin{equation}
\label{eqn:mm_3}
\begin{aligned}
\mathcal{J}_{\operatorname{F} \circ \operatorname{G}}\left[\Theta\right] &= \left( \nabla_{\Theta} \norm{\hat{\vPhi}_n\left( \Theta \right) - \bar{\vPhi}_n}^2_2 \right)\left[ \Theta \right] = \mathcal{J}_{\operatorname{F}}\left[ \operatorname{G}\left( \Theta \right) \right] \cdot \mathcal{J}_{\operatorname{G}}\left[ \Theta \right] 
\end{aligned}
\end{equation}

Next, we need the computation for $\mathcal{J}_{\operatorname{F}}\left[ \operatorname{G}\left( \Theta \right) \right]$ and $\mathcal{J}_{\operatorname{G}}\left[ \Theta \right]$ in Equation~\ref{eqn:mm_3}. 
First, we have: 
\begin{equation}\label{eqn:mm_4a}
\begin{split}
\mathcal{J}_{\operatorname{F}}\left[ \operatorname{G}\left( \Theta \right) \right] &= 2\left( \hat{\vPhi}_{n}\left( \Theta \right) - \bar{\vPhi}_{n}  \right) \\
&= 2\left( \hat{\vphi}_{n,1}\left( \Theta \right) - \bar{\vphi}_{n,1},\ \ldots\ , \hat{\vphi}_{n,j}\left( \Theta \right)  - \bar{\vphi}_{n,j},\ \ldots\  ,\hat{\vphi}_{n,m}\left( \Theta \right) - \bar{\vphi}_{n,m} \right), 
\end{split}
\end{equation}
where $\hat{\vPhi}_{n}\left( \Theta \right)$ and $\bar{\vPhi}_{n}$ are vectors of size $m$. 
And we also have:
\begin{equation}\label{eqn:mm_4b}
\resizebox{\hsize}{!}{$
\begin{aligned}
\mathcal{J}_{\operatorname{G}}\left[ \Theta \right] = 
\begin{pmatrix}
   \frac{\partial \mathbb{E}_{\vy \sim \operatorname{p}_{\Theta}\left(.|\vx_n\right)}\left[ \vphi_1\left( \vy|\vx_n \right) - \bar{\vphi}_{n,1} \right]}{\partial \Theta_1} & \ldots & \frac{\partial \mathbb{E}_{\vy \sim \operatorname{p}_{\Theta}\left(.|\vx_n\right)}\left[ \vphi_1\left( \vy|\vx_n \right) - \bar{\vphi}_{n,1} \right]}{\partial \Theta_{|\Theta|}} \\
   \ldots & \frac{\partial \mathbb{E}_{\vy \sim \operatorname{p}_{\Theta}\left(.|\vx_n\right)}\left[ \vphi_j\left( \vy|\vx_n \right) - \bar{\vphi}_{n,j} \right]}{\partial \Theta_i} & \ldots \\
   \frac{\partial \mathbb{E}_{\vy \sim \operatorname{p}_{\Theta}\left(.|\vx_n\right)}\left[ \vphi_m\left( \vy|\vx_n \right) - \bar{\vphi}_{n,m} \right]}{\partial \Theta_1} & \ldots & \frac{\partial \mathbb{E}_{\vy \sim \operatorname{p}_{\Theta}\left(.|\vx_n\right)}\left[ \vphi_m\left( \vy|\vx_n \right) - \bar{\vphi}_{n,m} \right]}{\partial \Theta_{|\Theta|}}
\end{pmatrix} 
= \mathcal{M}
\end{aligned}
$}
\end{equation}

A key part of these identities in Equation~\ref{eqn:mm_4b} is  the value of $\frac{\partial \mathbb{E}_{\vy \sim \operatorname{p}_{\Theta}\left(.|\vx_n\right)}\left[ \vphi_j\left( \vy|\vx_n \right) - \bar{\vphi}_{n,j} \right]}{\partial \Theta_i}$ which can be expressed as:
\begin{equation}\label{eqn:mm_5a}
\begin{split}
\mathcal{M}_{ij} = \frac{\partial \mathbb{E}_{\vy \sim \operatorname{p}_{\Theta}\left(.|\vx_n\right)}\left[ \vphi_j\left( \vy|\vx_n \right) - \bar{\vphi}_{n,j} \right]}{\partial \Theta_i} &= \frac{\partial \sum_{\vy} \operatorname{p}_{\Theta}\left( \vy|\vx_n \right) \left( \vphi_j\left( \vy|\vx_n \right) - \bar{\vphi}_{n,j} \right) }{\partial \Theta_i} \\
&= \sum_{\vy} \left( \vphi_j\left( \vy|\vx_n \right) - \bar{\vphi}_{n,j} \right) \frac{\partial \operatorname{p}_{\Theta}\left( \vy|\vx_n \right)}{\partial \Theta_i} \\
\end{split}
\end{equation}

Next, using the well-known ``log-derivative trick'':
\[
\operatorname{p}_{\Theta}\left( \vy|\vx \right) \frac{\partial \log \operatorname{p}_{\Theta}\left( \vy|\vx \right) }{\partial \Theta_i} = \frac{\partial \operatorname{p}_{\Theta}\left( \vy|\vx \right)}{\partial\Theta_i}
\]
from the Policy Gradient technique in reinforcement learning \citep{NIPS1999_1713}, we can rewrite the equation~\ref{eqn:mm_5a} as follows:
\begin{equation}\label{eqn:mm_5b}
\begin{split}
\mathcal{M}_{ij} &= \sum_{\vy} \left( \vphi_j\left( \vy|\vx_n \right) - \bar{\vphi}_{n,j} \right) \operatorname{p}\left( \vy|\vx_n \right) \frac{\partial \log \operatorname{p}_{\Theta}\left( \vy|\vx_n \right)}{\partial \Theta_i} \\
&= \mathbb{E}_{\vy \sim \operatorname{p}_{\Theta}\left(.|\vx_n\right)} \left[ \left( \vphi_j\left( \vy|\vx_n \right) - \bar{\vphi}_{n,j} \right) \frac{\partial \log \operatorname{p}_{\Theta}\left( \vy|\vx_n \right)}{\partial \Theta_i} \right].
\end{split}
\end{equation}

Combining Equations~\ref{eqn:mm_5a},~\ref{eqn:mm_5b}, we have:
\begin{equation*}
\begin{aligned}
&\frac{\partial \mathbb{E}_{\vy \sim \operatorname{p}_{\Theta}\left(.|\vx_n\right)}\left[ \vphi_j\left( \vy|\vx_n \right) - \bar{\vphi}_{n,j} \right]}{\partial \Theta_i} = 
\mathbb{E}_{\vy \sim \operatorname{p}_{\Theta}\left(.|\vx_n\right)} \left[ \left( \vphi_j\left( \vy|\vx_n \right) - \bar{\vphi}_{n,j} \right) \frac{\partial \log \operatorname{p}_{\Theta}\left( \vy|\vx_n \right)}{\partial \Theta_i} \right]\ ,
\end{aligned}
\end{equation*}
so in turn we obtain the computation of $\mathcal{J}_{\operatorname{G}}\left[ \Theta \right]$.  
Note that the expectation $\mathbb{E}_{\vy \sim \operatorname{p}_{\Theta}\left(.|\vx_n\right)}\left[ . \right]$ is easy to sample and the gradient $\frac{\partial \log \operatorname{p}_{\Theta}\left( \vy|\vx_n \right)}{\partial \Theta_i}$ is easy to evaluate as well. 

Since we already have the computations of $\mathcal{J}_{\operatorname{F}}\left[ \operatorname{G}\left( \Theta \right) \right]$ and $\mathcal{J}_{\operatorname{G}}\left[ \Theta \right]$, we can finalise the gradient computation $\Gamma_{\Theta,n}$ as follows:
\begin{equation*} 
\begin{aligned}
\Gamma_{\Theta,n} &= 2\left( \hat{\vphi}_{n,1}\left( \Theta \right) - \bar{\vphi}_{n,1},\ldots,\hat{\vphi}_{n,j}\left( \Theta \right) - \bar{\vphi}_{n,j},\ldots,\hat{\vphi}_{n,m}\left( \Theta \right) - \bar{\vphi}_{n,m} \right) \\ 
 &\cdot  \begin{pmatrix}
   \ldots & \ldots & \ldots \\
   \ldots &  \mathbb{E}_{\vy \sim \operatorname{p}_{\Theta}\left(.|\vx_n\right)} \left[ \left( \vphi_j\left( \vy|\vx_n \right) - \bar{\vphi}_{n,j} \right) \frac{\partial \log \operatorname{p}_{\Theta}\left( \vy|\vx_n \right)}{\partial \Theta_i} \right] & \ldots \\
   \ldots & \ldots & \ldots
\end{pmatrix} \\
&= 2\left( \hat{\vphi}_{n,1}\left( \Theta \right) - \bar{\vphi}_{n,1},\ldots,\hat{\vphi}_{n,j}\left( \Theta \right) - \bar{\vphi}_{n,j},\ldots,\hat{\vphi}_{n,m}\left( \Theta \right) - \bar{\vphi}_{n,m} \right) \\ 
 &\cdot \left[ \sum_{\vy} \operatorname{p}_{\Theta} \left( \vy|\vx_n \right)  \begin{pmatrix}
   \ldots & \ldots & \ldots \\
   \ldots &  \left( \vphi_j\left( \vy|\vx_n \right) - \bar{\vphi}_{n,j} \right) \frac{\partial \log \operatorname{p}_{\Theta}\left( \vy|\vx_n \right)}{\partial \Theta_i} & \ldots \\
   \ldots & \ldots & \ldots
\end{pmatrix} \right] \\
&= 2\sum_{\vy} \operatorname{p}_{\Theta} \left( \vy|\vx_n \right) \left( \hat{\vphi}_{n,1}\left( \Theta \right) - \bar{\vphi}_{n,1},\ldots,\hat{\vphi}_{n,j}\left( \Theta \right) - \bar{\vphi}_{n,j},\ldots,\hat{\vphi}_{n,m}\left( \Theta \right) - \bar{\vphi}_{n,m} \right)  \\ 
 &\cdot \begin{pmatrix}
   \ldots & \ldots & \ldots \\
   \ldots &  \left( \vphi_j\left( \vy|\vx_n \right) - \bar{\vphi}_{n,j} \right) \frac{\partial \log \operatorname{p}_{\Theta}\left( \vy|\vx_n \right)}{\partial \Theta_i} & \ldots \\
   \ldots & \ldots & \ldots
\end{pmatrix} \\
&= 2\sum_{\vy} \operatorname{p}_{\Theta} \left( \vy|\vx_n \right) \left\langle \hat{\vPhi}_{n}\left(\Theta \right) - \bar{\vPhi}_n, \vPhi\left( \vy|\vx_n \right) - \bar{\vPhi}_n  \right\rangle \nabla_{\Theta} \log \operatorname{p}_{\Theta} \left( \vy|\vx_n \right) \\
&= 2\mathbb{E}_{\vy \sim \operatorname{p}_{\Theta}\left(.|\vx_n\right)} \left[ \left\langle \hat{\vPhi}_{n}\left(\Theta \right) - \bar{\vPhi}_n, \vPhi\left( \vy|\vx_n \right) - \bar{\vPhi}_n  \right\rangle \nabla_{\Theta} \log \operatorname{p}_{\Theta} \left( \vy|\vx_n \right) \right].
\end{aligned}
\end{equation*}

By the reasoning just made, we can obtain the computation of $\Gamma_{\Theta,n}$ which is the central formula of the proposed moment matching technique. 
\end{proof}

\begin{table}[t]
\centering
\begin{tabular}{@{\extracolsep{4pt}}lcc}
\toprule   
{\bf Method} & {\bf Formulation} & {\bf Note} \\
\midrule
\multicolumn{3}{c}{\emph{Unconditional} Case} \\
\midrule
CE & $\nabla_{\Theta} \log \operatorname{p}_{\Theta} \left( \vy \right)$ &  $\vy \sim \mathcal{D}$ \\ 
RL w/ PG & $\mathcal{R} \left(\vy\right) \nabla_{\Theta} \log \operatorname{p}_{\Theta} \left( \vy \right)$  & $\vy \sim \operatorname{p}_{\Theta} \left( . \right)$ \\ 
MM & $\left\langle \hat{\vPhi}\left(\Theta \right) - \bar{\vPhi}, \vPhi\left( \vy \right) - \bar{\vPhi}  \right\rangle \nabla_{\Theta} \log \operatorname{p}_{\Theta} \left( \vy \right)$ & $\vy \sim \operatorname{p}_{\Theta} \left( . \right)$ \\ 
\midrule
\multicolumn{3}{c}{\emph{Conditional} Case} \\
\midrule
CE & $\nabla_{\Theta} \log \operatorname{p}_{\Theta} \left( \vy \given \vx \right)$ &  $\vx, \vy \sim \mathcal{D}$ \\ 
RL w/ PG & $\mathcal{R} \left(\vy\right) \nabla_{\Theta} \log \operatorname{p}_{\Theta} \left( \vy \given \vx \right)$  & $\vx \sim \mathcal{D}, \vy \sim \operatorname{p}_{\Theta} \left( . | \vx \right)$ \\ 
MM & $\left\langle \hat{\vPhi}\left(\Theta \right) - \bar{\vPhi}, \vPhi\left( \vy|\vx \right) - \bar{\vPhi}  \right\rangle \nabla_{\Theta} \log \operatorname{p}_{\Theta} \left( \vy \given \vx \right)$ & $\vx \sim \mathcal{D}; \vy \sim \operatorname{p}_{\Theta} \left( . | \vx \right)$ \\ 
\bottomrule
\end{tabular}
\caption{Comparing different methods for training seq2seq models. Note that we denote CE for cross-entropy, RL for reinforcement learning, PG as policy gradient, and MM for moment matching.} 
\label{tab:uc_ce_vs_rl_mm}
\end{table}

\subsection{MM training vs CE training vs RL training with Policy Gradient}
\label{sec:ce_rl_mm_compare}
Based on equation \ref{eqn:mm_6b}, and ignoring the constant factor, 
we can use as our gradient update, for each pair $\left( \vx_n, \vy_{j} \sim \operatorname{p}_{\Theta} \left( .|\vx_n \right) \right)$ ($j \in \left[ 1,J \right]$) the value
\[
\underbrace{\langle \hat{\vPhi}_{n,\vy}\left(\Theta \right) - \bar{\vPhi}_n, \vPhi\left( \vy|\vx_n \right) - \bar{\vPhi}_n  \rangle}_\text{multiplicative score} \underbrace{\nabla_{\Theta} \log \operatorname{p}_{\Theta} \left( \vy|\vx_n \right)}_\text{standard gradient update},
\]
where $\bar{\vPhi}_n$, the empirical average of $\vPhi\left( \cdot|\vx_n \right)$, can be estimated through the observed value $\vy_n$, i.e. $\bar{\vPhi}_n \simeq \vPhi\left( \vy_n|\vx_n \right)$.

Note that the above gradient update draws a very close connection to RL with policy gradient method \citep{NIPS1999_1713} where the ``multiplication score'' plays a similar role to the reward $\mathcal{R}\left(\vy|\vx\right)$; however, unlike RL training using a predefined reward 
, the major difference in MM training is that MM's multiplication score does depend on the model parameters $\Theta$ and looks at what the empirical data  tells the model via using explicit prior features. 
Table~\ref{tab:uc_ce_vs_rl_mm} compares the differences among three methods, namely CE, RL with Policy Gradient (PG) and our proposal MM, for neural seq2seq models in both unconditional (e.g., language modeling) and conditional (e.g., NMT, summarisation) cases. 

\subsection{Computing the Moment Matching Gradient}
We have derived the gradient of moment of matching loss as shown in Equation~\ref{eqn:mm_6b}. 
In order to compute it, we still need to have the evaluation of two estimates, namely the model average estimate $\hat{\vPhi}_{n}\left(\Theta \right)$ and the empirical average estimate $\bar{\vPhi}_n$. 

\paragraph{Empirical Average Estimate.} First, we need to estimate the empirical average $\bar{\vPhi}_n$. 
In the general case, given a source sequence $\vx_n$, suppose there are multiple target sequences $\vy \in \mathcal{Y}$ associated with $\vx_n$, then $\bar{\vphi}_n \equiv \frac{1}{|\mathcal{Y}|} \sum_{\vy \in \mathcal{Y}} \vphi\left( \vy|\vx_n \right)$. Specifically, when we have only one reference sequence $\vy$ per source sequence $\vx_n$, then $\bar{\vphi}_n \equiv \vphi\left( \vy|\vx_n \right)$ --- which is the standard case in the context of neural machine translation training. 

\paragraph{Model Average Estimate.} In practice, it is impossible to obtain a full computation of $\Gamma_{\Theta,n}$  due to intractable search of $\vy$. 
Therefore, we resort to estimate $\Gamma_{\Theta,n}$ by a sampling process. 
There are possible options for doing this. 

The {\bf simplistic} approach to that would be to:

First, estimate the model average $\hat{\vPhi}_{n} \left( \Theta \right)$ by sampling $\vy^{(1)},\vy^{(2)},\ldots,\vy^{(K)}$ and then estimating:
\[
\hat{\vPhi}_{n} \left( \Theta \right) \approx \frac{1}{K} \sum_{k \in \left[1,K\right]} \vPhi \left( \vy^{(k)} | \vx_n \right). 
\]
Next, estimate the expectation $\mathbb{E}_{\vy \sim \operatorname{p}_{\Theta}}$ in Equation~\ref{eqn:mm_6b} by independently sampling $J$ second values of $\vy$, and then estimate: 
\[
\Gamma_{\Theta,n} \approx \frac{1}{J} \sum_{j \in \left[1,J\right]} \left\langle \hat{\vPhi}_{n}\left(\Theta \right) - \bar{\vPhi}_n, \vPhi\left( \vy^{(j)} | \vx_n \right) - \bar{\vPhi}_n  \right\rangle \nabla_{\Theta} \log \operatorname{p}_{\Theta} \left( \vy^{(j)} | \vx_n \right).
\]
Note that two sets of samples are separate and $J \neq K$. This would provide an \underline{unbiased} estimate of $\Gamma_{\Theta,n}$, but at the cost of producing two independent sample sets of sizes $K$ and $J$, used for two different purposes --- which would be computationally wasteful. 

A more {\bf economical} approach might consist in using the \emph{same} sample set of size $J$ for both purposes. 
However, this would produce a \underline{biased} estimate of $\Gamma_{\Theta,n}$. 
This can be illustrated by considering the estimate case with $J=1$. 
In this case, the dot product $\left\langle \hat{\vPhi}_{n}\left(\Theta \right) - \bar{\vPhi}_n, \vPhi\left( \vy^{(1)} | \vx_n \right) - \bar{\vPhi}_n  \right\rangle$ in $\Gamma_{\Theta,n}$ is strictly positive since $\hat{\vPhi}_{n}\left(\Theta \right)$ is equal to $\vPhi\left( \vy^{(1)} | \vx_n \right)$; hence, in this case, the current sample $\vy^{(1)}$ to be systematically discouraged by the model. 

Here, we proposed a {\bf better} approach, resulting in an \textbf{unbiased} estimate of $\Gamma_{\Theta,n}$, formulated as follows:

First, we sample $J$ values of $\vy$: $\vy^{(1)},\vy^{(2)},\ldots,\vy^{(J)}$ with $J\geq2$, then: 
\begin{equation} \label{eqn:mm_7a}
\begin{aligned}
\Gamma_{\Theta,n} \approx \frac{1}{J} \sum_{j \in \left[1,J\right]} \left\langle \hat{\vPhi}_{n,j}\left(\Theta \right) - \bar{\vPhi}_n, \vPhi\left( \vy^{(j)} | \vx_n \right) - \bar{\vPhi}_n  \right\rangle \nabla_{\Theta} \log \operatorname{p}_{\Theta} \left( \vy^{(j)} | \vx_n \right),
\end{aligned}
\end{equation}
where 
\begin{equation} \label{eqn:mm_7b}
\begin{aligned}
\hat{\vPhi}_{n,j}\left(\Theta \right) \equiv \frac{1}{J-1} \sum_{\substack{j' \in \left[1,J\right] \\ j' \neq j}} \vPhi\left( \vy^{(j')} | \vx_n \right). 
\end{aligned}
\end{equation}

We can then prove that this computation provides an unbiased estimate of $\Gamma_{\Theta,n}$ (see \S\ref{proof_unbiasedness}).  
Note that here, we have exploited the same J samples for both purposes but have taken care of not exploiting the exact same $\vy^{(j)}$ for both --- akin to a Jackknife resampling estimator.\footnote{https://en.wikipedia.org/wiki/Jackknife\_resampling} 

\subsection{Proof of Unbiasedness}
For simplicity, we will consider here the ``unconditional case'', e.g.,  $\operatorname{p}\left( . \right)$ instead of $\operatorname{p}\left( . | \vx \right)$, but the conditional case follows easily.

\begin{lemma} \label{lm:lm_1}
Let $\operatorname{p} \left( . \right)$ be a probability distribution over $\vy$, and let $\zeta\left( \vy \right)$ be any function of $\vy$ and $\vPhi\left( \vy \right)$ is a feature vector over $\vy$. 

We wish to compute the quantity $\mathcal{A} = \mathbb{E}_{\vy \sim \operatorname{p}\left(.\right)} \left[ \langle \hat{\vPhi},\vPhi\left(\vy\right) \rangle \zeta\left(\vy\right) \right]$, where $\hat{\vPhi} \equiv \mathbb{E}_{\vy \sim \operatorname{p}\left(.\right)}\left[\vPhi\left(\vy\right)\right]$. 

Let us sample $J$ sequences $\vy^{(1)}$,\ldots,$\vy^{(J)}$ (where $J$ is a pre-defined number of generated samples) independently from $\operatorname{p}\left( . \right)$, and let us compute:
\begin{equation*} 
\begin{aligned}
\mathcal{B}\left( \vy^{(1)},\ldots,\vy^{(J)}  \right) \equiv \frac{1}{J} \left[ \langle \tilde{\vPhi}^{(-1)}, \vPhi\left(\vy^{(1)}\right) \rangle \zeta\left( \vy^{(1)} \right) + \ldots + \langle \tilde{\vPhi}^{(-J)}, \vPhi\left(\vy^{(J)}\right) \rangle \zeta\left( \vy^{(J)} \right) \right] ,
\end{aligned}
\end{equation*}
where $ \tilde{\vPhi}^{(-i)}$ is formulated as:
\[
\tilde{\vPhi}^{(-i)} \equiv \frac{1}{J-1} \left[ \vPhi\left(\vy^{(1)}\right) + \ldots + \vPhi\left(\vy^{(i-1)}\right) + \vPhi\left(\vy^{(i+1)}\right) + \ldots + \vPhi\left(\vy^{(J)}\right) \right].
\]

\noindent Then we have :
\begin{equation*} 
\mathcal{A} = \mathbb{E}_{\{ \vy^{(i)} \}_{i=1}^{J} \overset{iid}{\sim} \operatorname{p}\left(.\right)} \left[ \mathcal{B} \left( \vy^{(1)},\ldots,\vy^{(J)}  \right)  \right], 
\end{equation*}
in the other words, $\mathcal{B} \left( \vy_1,\ldots,\vy_J  \right)$ provides an \emph{unbiased} estimate of $\mathcal{A}$.
\end{lemma}

\begin{proof}
See Appendix. 
\end{proof}

To ground this lemma in our problem setting, consider the case where $\operatorname{p} = \operatorname{p}_{\Theta}$, and $\zeta(\vy) = \nabla_{\Theta} \log \operatorname{p}_{\Theta} \left( \vy \right)$, then the quantity $\mathcal{A} = \mathbb{E}_{\vy \sim \operatorname{p}\left(.\right)} \left[ \langle \hat{\vPhi},\vPhi\left(\vy\right) \rangle \zeta\left(\vy\right) \right]$ is equal  to the overall gradient of the MM loss, for a given value of the model parameters $\Theta$ (by the formula (\ref{eqn:mm_6b}) obtained earlier, and up to a constant factor). 
We would like to obtain an unbiased stochastic gradient estimator of this gradient, in other words, we want to obtain an unbiased estimator of the quantity $\mathcal{A}$. 

By Lemma~\ref{lm:lm_1}, $\mathcal{A}$ is equal to the expectation of $\mathcal{B} \left( \vy^{(1)},\ldots,\vy^{(J)}  \right)$, where $\vy^{(1)},\ldots,\vy^{(J)}$ are drawn i.i.d from distribution $\operatorname{p}$. 
In other words, if we sample one set of $J$ samples from $\operatorname{p}$, and compute $\mathcal{B} \left( \vy^{(1)},\ldots,\vy^{(J)} \right)$, where $\vy^{(1)},\ldots,\vy^{(J)}$ on this set, then \emph{we obtain an unbiased estimate} of $\mathcal{A}$. 
As a result, we obtain an unbiased estimate of the gradient of the overall MM loss, which is exactly what we need.

In principle, therefore, we need to first sample $\vy^{(1)},\ldots,\vy^{(J)}$, and to compute
\begin{equation*} 
\begin{aligned}
&\mathcal{B}\left( \vy^{(1)},\ldots,\vy^{(J)}  \right)  \\
&\equiv \frac{1}{J} \left[ \langle \tilde{\vPhi}^{(-1)}, \vPhi\left(\vy^{(1)}\right) \rangle 
\nabla_{\Theta} \log \operatorname{p}_{\Theta}
\left( \vy^{(1)} \right) + \ldots + \langle \tilde{\vPhi}^{(-J)}, \vPhi\left(\vy^{(J)}\right) \rangle 
\nabla_{\Theta} \log \operatorname{p}_{\Theta}
\left( \vy^{(J)} \right) \right], 
\end{aligned}
\end{equation*}

and then use this quantity as our stochastic gradient. 
In practice, what we do is to first sample $\vy^{(1)},\ldots,\vy^{(J)}$, and then use the components of the \emph{sum}:
$$
\langle \tilde{\vPhi}^{(-j)}, \vPhi\left(\vy^{(j)}\right) \rangle 
\nabla_{\Theta} \log \operatorname{p}_{\Theta}
\left( \vy^{(j)} \right)
$$
as our individual stochastic gradients. 
Note that this computation differs from the original one by a constant factor $\frac{1}{J}$, which can be accounted for by manipulating the learning rate. 

\subsection{Training with the Moment Matching Technique}
Recall the goal of our technique is to preserve certain aspects of generated target sequences according to prior knowledge. 
In principle, the technique does not teach the model how to generate a proper target sequence based on the given source sequence \citep{DBLP:journals/corr/RanzatoCAZ15}. 
For that reason, it has to be used along with standard CE training of seq2seq model. 
In order to train the seq2seq model with the proposed technique, we suggest to use one of two training modes: \emph{alternation} and \emph{interpolation}. 
For the \emph{alternation} mode, the seq2seq model is trained alternatively using both CE loss and moment matching loss. 
More specifically, the seq2seq model is initially trained with CE loss for some iterations, then switches to using moment matching loss; and vice versa. 
For the \emph{interpolation} mode, the model will be trained with the interpolated objective using two losses with an additional hyper-parameter balancing them. 
In summary, the general technique can be described as in Algorithm~\ref{alg:mm_algorithm}. 

\begin{algorithm*}[t]
\begin{algorithmic}[1]
 \State \textbf{Input}: a pre-trained model $\Theta$, parallel training data $\mathcal{D}$, $\lambda$ is balancing factor in interpolation training mode if used
 \For{ $step = 1, \ldots, \textrm{M}$} \Comment{ $M$ is maximum number of steps }
     \State Select a batch of size N source and target sequences in $\vX$ and $\vY$ in $\mathcal{D}$.      
     \If{MM mode is required}
     	\State Sample $J$ translations for the batch of source sequences $\vX$. \Comment{Random sampling has to be used.}
     	\State Compute the total MM gradients according to $\Gamma^{MM}_{\Theta,n} \equiv \mathbb{E}_{\vy \sim \operatorname{p}_{\Theta}\left(.|\vx_n\right)} \left[ \langle \hat{\vPhi}_{n,\vy}\left(\Theta \right) - \bar{\vPhi}_n, \vPhi\left( \vy|\vx_n \right) - \bar{\vPhi}_n  \rangle \nabla_{\Theta} \log \operatorname{p}_{\Theta} \left( \vy|\vx_n \right) \right]$ in Equations~\ref{eqn:mm_7a},~\ref{eqn:mm_7b}.
     \EndIf     
     \If{\emph{alternation mode}}
     	\If{MM mode}
     		\State Update model parameters according to the defined MM gradients $\Gamma^{MM}_{\Theta,n}$ with SGD. 
     	\ElsIf{CE mode}
     		\State Update model parameters according to standard CE based gradients $\Gamma^{CE}_{\Theta,n} \equiv \mathbb{E}_{\vx \sim \vX;\vy \sim \vY} \left[ \nabla_{\Theta} \log \operatorname{p}_{\Theta} \left( \vy|\vx \right) \right]$ with SGD as usual.  
     	\EndIf
     \Else \Comment{interpolation mode}
     	\State Compute the standard CE based gradients $\Gamma^{CE}_{\Theta,n} \equiv \mathbb{E}_{\vx \sim \vX;\vy \sim \vY} \left[ \nabla_{\Theta} \log \operatorname{p}_{\Theta} \left( \vy|\vx \right) \right]$.
     	\State Update model parameters according to $\Gamma^{interpolation}_{\Theta,n} \equiv \Gamma^{CE}_{\Theta,n} + \lambda \Gamma^{MM}_{\Theta,n}$. 
     \EndIf	
     \State After some steps, save model parameters w.r.t best score based on $\mathcal{J}^{dev}_{MM}$ using Equation~\ref{eqn:mm_13}.
 \EndFor
 \State  \Return newly-trained model $\Theta_{new}$ 
 \end{algorithmic}
\caption{General Algorithm for Training with Moment Matching Technique}
 \label{alg:mm_algorithm}
\end{algorithm*}

After some iterations of the algorithm, we can approximate $\mathcal{L}_{MM}$ over the development data (or sampled training data) through:
\begin{equation} \label{eqn:mm_13}
\begin{aligned}
\mathcal{J}^{dev}_{MM} \left( \Theta \right) \approx \frac{1}{N} \sum_{n=1}^{N} \norm{\hat{\vPhi}^{approx}_n\left( \Theta \right) - \bar{\vPhi}_n}^2_2.
\end{aligned}
\end{equation}
We expect $\mathcal{L}_{MM}$  to decrease over iterations, potentially improving the explicit evaluation measure(s), e.g., BLEU \citep{Papineni:2002:BMA:1073083.1073135} in NMT. 

%% file: rel_works.tex
\section{Connections to Previous Work}

\paragraph{Maximum Mean Discrepancies (MMD).}Our MM approach is related to the technique of Maximum Mean Discrepancies, a technique that has been successfully applied to computer vision, e.g., an alternative to learning generative adversarial network \citep{Li:2015:GMM:3045118.3045301,NIPS2017_6815}. 
The MMD is a way to measuring discrepancy between two distributions (for example, the empirical distribution and the model distribution) based on kernel-based similarities. The use of such kernels could potentially be useful in the long term to extend our approach, which can be seen as using a simple linear kernel over our pre-defined features, but in the specific context of seq2seq models, and in tandem with a generative process based on an auto-regressive generative model.

\paragraph{The Method of Moments.}Recently, \cite{DBLP:conf/icml/RavuriMRV18} proposed using a moment matching technique in situations where Maximum Likelihood is difficult to apply.
A strong difference with the way we use MM is that they define feature functions parameterised by some parameters and let them be learned along with model parameters. In fact, they are trying to applying the method of moments to situations in which ML (maximum likelihood, or CE) is not applicable, but where MM can find the correct model distribution on its own. Hence the focus on having (and learning) a large number of features, because only many features will allow to approximate the actual distribution. In our case, we are not relying on MM to model the target distribution on its own. Doing so with a small number of features would be doomed (e.g, thinking of the length ratio feature: it would only guarantee that the translation has a correct length, irrespective of the lexical content). We are using MM to complement ML, in such a way that task-related important features are attended to even if that means getting a (slightly) worse likelihood (or perplexity) on the training set. One can in fact see our use of MM as a form of regularization technique for complementing the MLE training and this is an important aspect of our proposed MM approach.

%% file: exp.tex
\section{Preliminary Experiments}
\subsection{Prior Features for NMT}
In order to validate the proposed technique, we re-applied two prior features used for training NMT as in \citep{P17-1139}, including source and target length ratio and lexical bilingual features. \cite{P17-1139} showed in their experiments that these two are the most effective features for improving NMT systems. 

The first feature is straightforward, just about measuring the ratio between source and target length. 
This feature aims at forcing the model to produce translations with consistent length ratio between source and target sentences, in such a way that too short or too long translations will be avoided. 

Given the respective source and target sequences $\vx$ and $\vy$, we define this source and target length ratio feature function $\vPhi_{len\_ratio}$ as follows:
\begin{equation}
\label{eqn:mm_fea_1a}
\vPhi_{len\_ratio} := 
\begin{cases} 
\frac{\beta * |\vx|}{|\vy|} \quad \text{if} \quad \beta \times |\vx| < |\vy|  \\  \frac{|\vy|}{|\beta * |\vx|} \quad \text{otherwise} 
\end{cases},
\end{equation}
where $\beta$ is additional hyper-parameter, normally set empirically based on prior knowledge about source and target languages. In this case, the feature function is a real value. 

The second feature we used is based on a word-to-word lexical translation dictionary produced by an off-the-shelf SMT system (e.g., Moses).\footnote{https://github.com/moses-smt/mosesdecoder} 
The goal of this feature is to ask the model to take external lexical translations into consideration. This feature will be potentially useful in cases such as: translation for rare words, and in low resource setting in which parallel data can be scarce.\footnote{NMT has been empirically found to be less robust in such a setting than SMT. }  
 Following \cite{P17-1139}, we defined sparse feature functions  
 \[
 \vPhi_{bd} \equiv \left[ \vphi_{\langle w_{x_1}, w_{y_1} \rangle},\ldots, \vphi_{\langle w_{x_i}, w_{y_j} \rangle}, \ldots, \vphi_{\langle w_{x_{\mathcal{D}_{lex}}}, w_{y_{\mathcal{D}_{lex}}} \rangle} \right],
 \]
 where:
 \begin{equation*} \label{eqn:mm_fea_1b}
 \vphi_{\langle w_x, w_y \rangle} := 
 \begin{cases} 
 1 \quad \text{if} \quad w_x \in \vx \wedge w_y \in \vy  \\  0 \quad \text{otherwise} 
 \end{cases},
\end{equation*}
 and where $D_{lex}$ is a lexical translation dictionary produced by Moses. 



\subsection{Datasets and Baseline}
We proceed to validate the proposed technique with small-scale experiments. 
We used the IWSLT'15 dataset, translating from English to Vietnamese. 
This dataset is relatively small, containing approximately 133K sentences for training, 1.5K for development , and 1.3K for testing. 
We re-implemented the transformer architecture \citep{NIPS2017_7181} for training our NMT model\footnote{in our open source toolkit: https://github.com/duyvuleo/Transformer-DyNet}  with hyper-parameters: 4 encoder and 4 decoder layers; hidden dimension 512 and dropout probability 0.1 throughout the network.
For the sampling process, we generated 5 samples for each moment matching training step. 
We used interpolation training mode with a balancing hyper-parameter of 0.5. In fact, changing this hyper-parameter only slightly affects the overall result. 
For the feature with length ratio between  source and target sequences, we used the length factor $\beta=1$. 
For the feature with bilingual lexical dictionary, we extracted it by Moses's training scripts. 
In this dictionary, we filtered out the bad entries based on word alignment probabilities produced by Moses, e.g., using a threshold less than 0.5 following \cite{P17-1139}. 

\subsection{Results and Analysis}
\begin{table}[t]
\centering
\begin{tabular}{@{\extracolsep{4pt}}lcc}
\toprule   
{} & {\bf BLEU} & {\bf MM Loss} \\
\midrule
tensor2tensor \citep{NIPS2017_7181} & 27.69 &  \_ \\ 
base (our reimplementation - Transformer-DyNet)  & 28.53 &  0.0094808 \\ 
base+mm &  {\bf 29.17{\bestss}} &  {\bf 0.0068775} \\ 
\bottomrule
\end{tabular}
\caption{Evaluation scores for training moment matching with length ratio between source and target sequences; \textbf{bold}: statistically significantly better than the baselines, \bestss: best performance on dataset.} 
\label{tab:result_lr}
\end{table}
\begin{table}[t]
\centering
\begin{tabular}{@{\extracolsep{4pt}}lcc}
\toprule   
{} & {\bf BLEU} & {\bf MM Loss} \\
\midrule
tensor2tensor  \citep{NIPS2017_7181} & 27.69 &  \_ \\ 
base (our reimplementation - Transformer-DyNet) & 28.53 &  0.7384 \\ 
base+mm & {\bf 29.11{\bestss}} &  {\bf 0.7128} \\ 
\bottomrule
\end{tabular}
\caption{Evaluation scores for training moment matching with bilingual lexical dictionary; \textbf{bold}: statistically significantly better than the baselines, \bestss: best performance on dataset.} 
\label{tab:result_bd}
\end{table}
Our results can be found in Table~\ref{tab:result_lr} and~\ref{tab:result_bd}. 
As can be seen from the tables, as long as the model attempted to reduce the moment matching loss, the BLEU scores \citep{Papineni:2002:BMA:1073083.1073135} improved statistically significantly with $p<0.005$ \citep{koehn:2004:EMNLP}. 
This was consistently shown  in both experiments as an encouraging validation of our  proposed training technique with moment matching. 

%% file: concl.tex
\section{Conclusion}
We have shown some nice mathematical properties of the proposed moment matching training technique (in particular, unbiasedness) and believe it is promising. 
Our initial experiments indicate its potential for improving existing NMT systems using simple prior features. 
Future work may include exploiting more advanced features for improving NMT and evaluate our proposed technique on larger-scale datasets. 

%% file: supp.tex
\newpage

\section*{Appendix - Proof of Unbiasedness}
\label{sec:supplemental}
\begin{proof}
  Let us define:
\begin{equation} \label{ch7:eqn:mm_10}
\begin{aligned}
\mathcal{E} 
  &= \mathbb{E}_{\{ \vy^{(i)} \}_{i=1}^{J} \overset{iid}{\sim} \operatorname{p}\left(.\right)} \left[ \mathcal{B} \left( \vy^{(1)},\ldots,\vy^{(J)}  \right)  \right] \\
  &= \begin{aligned}  
  \sum_{\vy^{(1)},\ldots,\vy^{(J)}} \operatorname{p}\left(\vy^{(1)}\right)\ldots\operatorname{p}\left(\vy^{(J)}\right) \frac{1}{J} \Bigl[\langle \tilde{\vPhi}^{(-1)}, \vPhi\left(\vy^{(1)}\right) \rangle \zeta\left( \vy^{(1)} \right) + \ldots &+ \langle \tilde{\vPhi}^{(-i)}, \vPhi\left(\vy^{(i)}\right) \rangle \zeta\left( \vy^{(i)} \right) + \ldots \\
  &\ldots + \langle \tilde{\vPhi}^{(-J)}, \vPhi\left(\vy^{(J)}\right) \rangle \zeta\left( \vy^{(J)} \right)  \Bigr]
  \end{aligned} \\
  &= \frac{1}{J} \sum_{i\in[1,J]}  \sum_{\vy^{(1)},\ldots,\vy^{(J)}} \operatorname{p}\left(\vy^{(1)}\right)\ldots\operatorname{p}\left(\vy^{(J)}\right)
  \langle \tilde{\vPhi}^{(-i)}, \vPhi\left(\vy^{(i)}\right) \rangle \zeta\left( \vy^{(i)} \right)
\end{aligned}
\end{equation}

For a given value of $i$, we have:
\begin{equation}
\label{ch7:eqn:mm_11}
\begin{aligned}
& \sum_{\vy^{(1)},\ldots,\vy^{(J)}} \operatorname{p}\left(\vy^{(1)}\right)\ldots\operatorname{p}\left(\vy^{(J)}\right)  \langle \tilde{\vPhi}^{(-i)}, \vPhi\left(\vy^{(i)}\right) \rangle \zeta\left( \vy^{(i)} \right) \\
&\begin{aligned}
= \sum_{\vy^{(i)}} \operatorname{p}\left(\vy^{(i)}\right) \zeta\left( \vy^{(i)} \right) &\sum_{\vy^{(1)},\ldots,\vy^{(i-1)},\vy^{(i+1)},\ldots,\vy^{(J)}} \operatorname{p}\left(\vy^{(1)}\right)\ldots\operatorname{p}\left(\vy^{(i-1)}\right)\operatorname{p}\left(\vy^{(i+1)}\right)\ldots\operatorname{p}\left(\vy^{(J)}\right) \frac{1}{J-1} \\
& \langle \left[ \vPhi\left(\vy^{(1)}\right)+\ldots+\vPhi\left(\vy^{(i-1)}\right)+\vPhi\left(\vy^{(i+1)}\right)+\ldots+\vPhi\left(\vy^{(J)}\right) \right],\vPhi\left(\vy^{(i)}\right) \rangle 
\end{aligned} \\
&\begin{aligned}
= \sum_{\vy^{(i)}} &\operatorname{p}\left(\vy^{(i)}\right) \zeta\left( \vy^{(i)} \right) \frac{1}{J-1} \Bigl[ \\
																														   &\sum_{\vy^{(1)},\ldots,\vy^{(i-1)},\vy^{(i+1)},\ldots,\vy^{(J)}} \operatorname{p}\left(\vy^{(1)}\right)\ldots\operatorname{p}\left(\vy^{(i-1)}\right)\operatorname{p}\left(\vy^{(i+1)}\right)\ldots\operatorname{p}\left(\vy^{(J)}\right) \langle \vPhi\left(\vy^{(1)}\right), \vPhi\left(\vy^{(i)}\right) \rangle + \\
																															&\ldots \\
																															&+ \sum_{\vy^{(1)},\ldots,\vy^{(i-1)},\vy^{(i+1)},\ldots,\vy^{(J)}} \operatorname{p}\left(\vy^{(1)}\right)\ldots\operatorname{p}\left(\vy^{(i-1)}\right)\operatorname{p}\left(\vy^{(i+1)}\right)\ldots\operatorname{p}\left(\vy^{(J)}\right) \langle \vPhi\left(\vy^{(J)}\right), \vPhi\left(\vy^{(i)}\right) \rangle  \Bigr]
\end{aligned} \\
&\begin{aligned}
= \sum_{\vy^{(i)}} \operatorname{p}\left(\vy^{(i)}\right) \zeta\left( \vy^{(i)} \right) \frac{1}{J-1} \Bigl[ &\sum_{\vy^{(1)}} \operatorname{p}\left(\vy^{(1)}\right) \langle \vPhi\left(\vy^{(1)}\right), \vPhi\left(\vy^{(i)}\right) \rangle + \\ 
																																				&\ldots \\
																																				&+ \sum_{\vy^{(i-1)}} \operatorname{p}\left(\vy^{(i-1)}\right) \langle \vPhi\left(\vy^{(i-1)}\right), \vPhi\left(\vy^{(i)}\right) \rangle \\
																																				&+ \sum_{\vy^{(i+1)}} \operatorname{p}\left(\vy^{(i+1)}\right) \langle \vPhi\left(\vy^{(i+1)}\right), \vPhi\left(\vy^{(i)}\right) \rangle \\ 
																																				&\ldots \\
																																				&+ \sum_{\vy^{(J)}} \operatorname{p}\left(\vy^{(J)}\right) \langle \vPhi\left(\vy^{(J)}\right), \vPhi\left(\vy^{(i)}\right) \rangle \Bigr]
\end{aligned} \\
&= 	\sum_{\vy^{(i)}} \operatorname{p}\left(\vy^{(i)}\right) \zeta\left( \vy^{(i)} \right)  \frac{1}{J-1} \left[ \underbrace{\langle \hat{\vPhi}, \vPhi\left(\vy^{(i)}\right) \rangle + \ldots + \langle \hat{\vPhi}, \vPhi\left(\vy^{(i)}\right) \rangle}_\text{$J-1$ times}  \right] \\
&= \sum_{\vy^{(i)}} \operatorname{p}\left(\vy^{(i)}\right) \zeta\left( \vy^{(i)} \right)  \langle \hat{\vPhi}, \vPhi\left(\vy^{(i)}\right) \rangle \\
&= \mathbb{E}_{\vy \sim \operatorname{p}\left(.\right)}  \langle \hat{\vPhi}, \vPhi\left(\vy\right) \rangle  \zeta\left( \vy \right).
\end{aligned}
\end{equation}

\noindent Finally, by collecting the results for the $J$ values of the index $i$, we obtain:
\begin{equation}
\label{ch7:eqn:mm_12}
\begin{aligned}
\mathcal{E} &= \frac{1}{J} \cdot J \cdot  \mathbb{E}_{\vy \sim \operatorname{p}\left(.\right)}  \langle \hat{\vPhi}, \vPhi\left(\vy\right) \rangle  \zeta\left( \vy \right) \\
&= \mathbb{E}_{\vy \sim \operatorname{p}\left(.\right)}  \langle \hat{\vPhi}, \vPhi\left(\vy\right) \rangle  \zeta\left( \vy \right)\\
&= \mathcal{A}.
\end{aligned}
\end{equation}
\end{proof}
